\documentclass{article}
\usepackage{times}
\usepackage{url}
\usepackage{amsmath}
\usepackage{algorithm}
\usepackage{algorithmic}
\usepackage{graphicx}
\usepackage{subfigure}
\usepackage{wrapfig}
\usepackage{colortbl}
\usepackage{javen}

\title{Conditional Random Fields and Support Vector Machines: A Hybrid Approach}

\author{
Qinfeng Shi\\
The University of Adelaide\\
Adelaide, SA, Australia\\
\texttt{qinfeng.shi@ieee.org}
\and Mark Reid\\
The Australian National University\\
Canberra, Australia\\
 \texttt{mark.reid@anu.edu.au}\\
\and Tiberio Caetano\\
The Australian National University and NICTA\\
Canberra, Australia\\
 \texttt{Tiberio.Caetano@nicta.com.au}
}

\date{\today}

\begin{document}

\maketitle


\begin{abstract}
We propose a novel hybrid loss for multiclass and structured prediction
problems that is a convex combination of log loss for Conditional
Random Fields (CRFs) and a multiclass hinge loss for Support Vector
Machines (SVMs). We provide a sufficient condition for when the
hybrid loss is Fisher consistent for classification. This condition 
depends on a measure of dominance between labels -- specifically, 
the gap in per observation probabilities between the most likely 
labels. We also prove Fisher consistency is necessary for parametric consistency
when learning models such as CRFs.

We demonstrate empirically that the hybrid loss typically performs
as least as well as -- and often better than -- both of its constituent losses on 
variety of tasks. In doing so we also provide an empirical comparison of the
efficacy of probabilistic and margin based approaches to multiclass
and structured prediction and the effects of label dominance on
these results.
\end{abstract}

\section{Introduction} Conditional Random Fields (CRFs) and Support
Vector Machines (SVMs) can be seen as representative of two
different approaches to classification problems. The former is
purely \emph{probabilistic} -- the conditional probability of
classes given each observation is explicitly modelled -- while the
latter is purely \emph{discriminative} -- classification is
performed without any attempt to model probabilities.
Both approaches have their strengths and weaknesses. CRFs~\cite{LafMcCPer01,ShaPer03} 
are known to yield the Bayes optimal solution asymptomatically but often require a
large number of training examples to do accurate modelling. In
contrast, SVMs make more efficient use of
training examples but are known to be inconsistent when there are
more than two classes~\cite{TewBar07,Liu07}.

Despite their differences, CRFs and SVMs appear very
similar when viewed as optimisation problems. The most salient
difference is the loss used by each: CRFs are trained using a log
loss while SVMs typically use a hinge loss.
In an attempt to capitalise on their relative strengths and avoid
their weaknesses, we propose a novel
\emph{hybrid loss} which ``blends'' the two losses. 
After some background (\S\ref{sec:background}) we provide the following analysis:
We argue that Fisher Consistency
for Classification (FCC) -- a.k.a. classification calibration -- 
is too coarse a notion and introduce a
distribution-dependent refinement called Conditional Fisher Consistency
for Classification (\S\ref{sec:consistency}).
We prove the hybrid loss is conditionally FCC and give a noise condition that relates 
the hybrid loss's mixture parameter to a margin-like property of the data distribution 
(\S\ref{sec:cfcc}).
We then show that, although FCC is effectively a non-parametric condition, it is
also a necessary condition for consistent risk minimisation using parametric models 
(\S\ref{sec:parametric-fcc}).
Finally, we empirically test the hybrid loss on 
various domains including multiclass classification, Chunking and Named Entity 
Recognition and show it consistently performs better than either of its constituent 
losses (\S\ref{sec:experiments}).


\section{Losses for Multiclass Prediction}
\label{sec:background} In classification problems
\emph{observations} $x\in\Xcal$ are paired with \emph{labels}
$y\in\Ycal$ via some joint distribution $D$ over $\Xcal\times\Ycal$.
We will write $D(x,y)$ for the joint probability and $D(y|x)$ for
the conditional probability of $y$ given $x$. Since the labels $y$
are finite and discrete we will also use the notation $D_y(x)$ for
the conditional probability to emphasise that distributions over
$\Ycal$ can be thought of as vectors in $\RR^k$ for $k=|\Ycal|$.
We will use $q$ to denote distributions over $\Ycal$
when the observations $x\in\Xcal$ are irrelevant.

When the number of possible labels $k=|\Ycal|>2$ we call the
classification problem a \emph{multiclass} classification problem. A
special case of this type of problem is \emph{structured prediction}
where the set of labels $\Ycal$ has some combinatorial structure
that typically means $k$ is very large ~\cite{BakHofSchSmoetal07}.\footnote{
    In structured prediction, each output $y$ involves relationships among 
    `sub-components' of $y$. For example, the label of a pixel in an image depends on 
    the label of neighbouring pixels. That's where the term `structured' comes from. 
    However, different $y$'s are typically \emph{not} assumed to possess any joint 
    structure (i.e., it is typically assumed that the data is drawn from 
    $\Xcal \times \Ycal$). This is why structured prediction is no different in essence 
    than multiclass classification.}
As seen in the experimental section below a variety of problems,
such as text tagging, can be construed as structured prediction
problems.

Given $m$ training observations $S=\{(x_i,y_i)\}_{i=1}^m$ drawn i.i.d.
from $D$, the aim of the learner is to produce a predictor $h : \Xcal
\to \Ycal$ that minimises the \emph{misclassification error} $e_D(h)
= \PP_D\left[ h(x) \ne y \right]$. Since the true distribution is
unknown, an approximate solution to this problem is typically found
by minimising a regularised empirical estimate of the risk for a
\emph{surrogate loss} $\ell$. Examples of surrogate losses will be
discussed below.

Once a loss is specified, a solution is found by solving
\begin{equation}\label{eq:erm}
    \min_{f} \frac{1}{m}\sum_{i=1}^m \ell(f(x_i),y_i) + \Omega(f)
\end{equation}
where each \emph{model} $f : \Xcal \to \RR^k$ assigns a vector of
scores $f(x)$ to each observation and the regulariser $\Omega(f)$
penalises overly complex functions. A model $f$ found in this way
can be transformed into a predictor by defining $h_f(x) =
\argmax_{y\in\Ycal} f_y(x)$. We will overload the definition of 
misclassification error and sometimes write $e_D(f)$ as shorthand
for $e_D(h_f)$.

In structured prediction, the models are usually specified in terms
of a parameter vector $w\in\RR^n$ and a feature map $\phi :
\Xcal\times\Ycal \to \RR^n$ by defining $f_y(x;w) =
\inner{w}{\phi(x,y)}$ and in this case the regulariser is
$\Omega(f) = \frac{\lambda}{2}\|w\|^2$ for some choice of
$\lambda\in\RR$. This is the framework used to implement the SVMs
and CRFs used in the experiments described in Section~\ref{sec:experiments}.
Although much of our analysis does not assume any particular parametric model,
we explicitly discuss the implications of doing so in \S\ref{sec:parametric-fcc}.

A common surrogate loss for multiclass problems is a generalisation
of the binary class hinge loss used for Support Vector
Machines~\cite{CraSin00}:
\begin{equation}\label{eq:hinge}
    \ell_H(f,y) = \left[1 - M(f,y)\right]_+
\end{equation}
where $[z]_+ = z$ for $z>0$ and is 0 otherwise, and $M(f,y) = f_y -
\max_{y'\ne y}f'_y$ is the \emph{margin} for the vector $f\in\RR^k$.
Intuitively, the hinge loss is minimised by models that not only
classify observations correctly but also maximise the difference
between the highest and second highest scores assigned to the
labels. 

While there are other, consistent losses for SVMs~\cite{TewBar07,Liu07}, these cannot scale up to structured
estimations due to computational issues.
For example, the multiclass hinge loss $\sum_{j\neq y}[1+f_j(x)]_+$ is shown to be 
consistent in \cite{Liu07}. However, it requires evaluating $f$ on all possible labels 
except the true $y$. This is intractable for structured estimation where the possible labels grow exponentially with the size of the structured output.
Since the other known and consistent multiclass hinge losses have similar 
intractability we will only focus on the margin-based loss $\ell_H$ which can be
evaluated quickly using techniques from dynamic programming, linear programming \emph{etc.} \cite{TsoJoaHofAlt05, TasGueKol04, BakHofSchSmoetal07}.

\subsection{Probabilistic Models and Losses}
\label{sec:background-pmodel} The scores given to labels by a
general model $f:\Xcal\to\RR^k$ can be transformed into a
conditional probability distribution $p(x;f)\in[0,1]^k$ by letting
\begin{equation}\label{eq:prob}
    p_y(x;f) = \frac{\exp(f_y(x))}{\sum_{y\in\Ycal}{\exp(f_y(x))}}.
\end{equation}

It is easy to show that under this interpretation the hinge loss for
a probabilistic model $p = p(\cdot;f)$ is given by
\begin{equation*}
    \ell_H(p,y) = \left[1-\ln\frac{p_y}{\max_{y'\ne y}p_{y'}}\right]_+
\end{equation*}

Another well known loss for probabilistic models, such as
CRFs, is the log loss
\begin{equation*}
    \ell_L(p,y) = -\ln p_y.
\end{equation*}
This loss penalises models that assign low probability to likely
instances labels and, implicitly, that assign high probability to
unlikely labels.

We now propose a novel \emph{hybrid loss} for probabilistic models that is a convex 
combination of the hinge and log losses
\begin{equation}
    \ell_\alpha(p,y)
    = \alpha\ell_L(p,y) + (1-\alpha)\ell_H(p,y) \label{eq:hybrid}
\end{equation}
where mixture of the two losses is controlled by a parameter
$\alpha\in[0,1]$. Setting $\alpha = 1$ or $\alpha = 0$ recovers the
log loss or hinge loss, respectively. The intention is that choosing
$\alpha$ close to 0 will emphasise having the maximum gap between the largest and 
second largest label probabilities while an $\alpha$ close to 1 will
force models to prefer accurate probability assessments over strong
classification.

\section{Fisher Consistency For Classification}\label{sec:consistency}
A desirable property for a loss is that, given enough data, the models
obtained by minimising the loss at each observation will make predictions
that are consistent with the true label probabilities at each observation.

Formally, we say vector $f\in\RR^{|\Ycal|}$ is \emph{aligned} with
a distribution $q$ over $\Ycal$ whenever maximisers of $f$ are
also maximisers for $q$. That is, when
\(
    \argmax_{y\in\Ycal} f_y \subseteq \argmax_{y\in\Ycal} q_y.
\)
If, for all label distributions $q$, minimising the conditional risk 
$L(f) = \EE_{y\sim q}[\ell(f,y)]$ for a loss $\ell$ yields a vector $f^*$ aligned
with $q$ we will say $\ell$ is \emph{Fisher consistent for
classification} (FCC) \footnote{
    Note that the Fisher consistency for classification is weaker than Fisher 
    consistency for density estimation. The former requires the same prediction only, 
    while the latter requires the estimated density is the same as the true data
    distribution. In this paper, we focus on the former only.} 
-- or \emph{classification calibrated} \cite{TewBar07}.
This is an important property for losses since it is
equivalent to the asymptotic consistency of the empirical risk
minimiser for that loss \cite[Theorem 2]{TewBar07}.

The standard multiclass hinge loss $\ell_H$ is known to be inconsistent for
classification when there are more than two
classes~\cite{Liu07,TewBar07}. The analysis in \cite{Liu07} shows
that the hinge loss is inconsistent whenever there is an instance $x$ with
a \emph{non-dominant} distribution -- that is, $D_y(x) <
\frac{1}{2}$ for all $y\in\Ycal$. Conversely, A distribution is
\emph{dominant} for an instance $x$ if there is some $y$ with
$D_y(x) > \frac{1}{2}$.
In contrast, the log loss used to train non-parametric CRFs is
Fisher consistent for probability estimation -- that is, the
associated risk is minimised by the true conditional distribution --
and thus $\ell_C$ is FCC since the minimising distribution is equal
to $D(x)$ and thus aligned with $D(x)$. 

\subsection{Conditional Consistency of the Hybrid Loss}
\label{sec:cfcc}

In order to analyse the consistency of the hybrid loss we introduce
a more refined notion of Fisher consistency that takes into account
the true distribution of class labels. If $q=(q_1,\ldots,q_k)$ is a
distribution over the labels $\Ycal$ then we say the loss $\ell$ is
\emph{conditionally FCC with respect to $q$} whenever minimising the
conditional risk w.r.t. $q$, $L_q(f) = \EE_{y\sim
q}\left[\ell(f,y)\right]$ yields a predictor $f^*$ that is
consistent with $q$. Of course, if a loss $\ell$ is conditionally
FCC w.r.t. $q$ for \emph{all} $q$ it is, by definition,
(unconditionally) FCC.

\begin{theorem}\label{thm:fcc}
    Let $q=(q_1,\ldots,q_k)$ be a distribution over labels and let
    $y_1 = \max_y q_y$ and $y_2 = \max_{y\ne y_1} q_y$ be the two most likely
    labels.
    Then the hybrid loss $\ell_\alpha$ is conditionally FCC for $q$ whenever
    $q_{y_1} > \frac{1}{2}$ or
    \begin{equation}\label{eq:alpha}
        \alpha > 1 - \frac{q_{y_1} - q_{y_2}}{1-2q_{y_1}}.
    \end{equation}
\end{theorem}
For the proof see Appendix \ref{sec:proof-fcc}.
Theorem~\ref{thm:fcc} can be inverted and interpreted as a
constraint on the conditional distributions of some data distribution $D$ such that a 
hybrid loss with parameter $\alpha$ will yield consistent predictions. Specifically,
the hybrid loss will be consistent if, for all $x\in\Xcal$ such that
$q = D(x)$ has no dominant label (\emph{i.e.}, $D_y(x) \le
\frac{1}{2}$ for all $y\in\Ycal$), the gap $D_{y_1}(x) -
D_{y_2}(x)$ between the top two probabilities is larger than
$(1-\alpha)(1-2D_{y_1}(x))$. When this is not the case for some $x$,
the classification problem for that instance is, in some sense, too
difficult to disambiguate. In this sense, the bound can be seen as a property
on distributions akin to Tsybakov's noise condition 
\cite{Che06}. Making this analogy precise is the focus of ongoing work.

\subsection{Parametric Consistency}\label{sec:parametric-fcc}

Since Fisher consistency is defined point-wise on observations, it is not directly
applicable to parametric models as these enforce inter-observational constraints 
(\emph{e.g.} smoothness).
Abstractly, assuming parametric hypotheses can be seen as a restriction over
the space of allowable scoring functions. When learning parametric models, risks
are minimised over some subset $\Fcal$ of functions from $\Xcal \to \RR^{\Ycal}$
instead of all possible functions.
We now show that, given some weak assumptions on the hypothesis class $\Fcal$, a 
loss being FCC is a necessary condition if the loss is also to be 
$\Fcal$-consistent.

We say a loss $\ell$ is $\Fcal$-consistent if, for any distribution, minimising its
associated risk over $\Fcal$ yields a hypothesis with minimal
0-1 loss in $\Fcal$.\footnote{
    While this is simpler and stronger than the usual asymptotic notation of 
    consistency \cite{LugVay04} it most readily relates to FCC and suffices for our 
    discussion since we are only establishing that FCC is a necessary condition.}
Recall that the risk of a hypothesis $f\in\Fcal$ associated with a loss $\ell$ and
distribution $D$ over $\Xcal\times\Ycal$ is $L_D(f) = \EE_D\left[\ell(y,f(x))\right]$
and its 0-1 risk or misclassification error is
$e_D(f) = \PP_D\left[y\ne \argmax_{y'\in\Ycal} f_{y'}(x)\right]$. Formally then, given
a function class $\Fcal$ we say \emph{$\ell$ is $\Fcal$-consistent} if, for all
distributions $D$,
\begin{equation}\label{eq:F-consistency}
    L_D(f^*) = \inf_{f\in\Fcal} L_D(f)
    \implies
    e_D(f^*) = \inf_{f\in\Fcal} e_D(f).
\end{equation}

We need a relatively weak condition on function classes $\Fcal$ to state our theorem.
We say a class $\Fcal$ is \emph{regular} if the follow two properties hold: 1) For any
$g\in\RR^{\Ycal}$ there exists an $x\in\Xcal$ and an $f\in\Fcal$ so that $f(x) = g$;
and 2) For any $x\in\Xcal$ and $y\in\Ycal$ there exists an $f\in\Fcal$ so that
$y = \argmax_{y'\in\Ycal} f_{y'}(x)$.
Intuitively, the first condition says that for any distribution over labels
there must be a function in the class which models it perfectly on some point in the
input space.
The second condition requires that any mode can be modelled on any input.
Importantly, these properties are fairly weak in that they do not say anything about the
constraints a function class might put on relationships between distributions modelled
on different inputs.

\begin{theorem}\label{thm:fcc-necessary}
    For regular function classes $\Fcal$ any loss that is $\Fcal$-consistent is
    necessarily also Fisher Consistent for Classification (FCC).
\end{theorem}
The full proof is in Appendix~\ref{sec:proof-fcc-nec}. The argument sketch is: 
since $\Fcal$-consistency requires (\ref{eq:F-consistency}) to hold for all $D$ it
must hold for a $D$ with all of its mass on a single observation $x_0$. 
If $\ell$ is not FCC there must be some label distribution $q$ and vector $g$ so that 
$L_q(g)$ is minimal but $e_q(g)$ is not. Choosing 
$x_0$ so that $f(x_0) = g$ (by the regularity of $\Fcal$) and setting 
$D(y|x) = q$ gives a contradiction.
 
\subsection{Generalisation Bound }\label{sec:genbound}
We now give a PAC-Bayesian bound \cite{McAllester98} for the generalisation error 
$e_D$ of the hybrid model that can be specialised to recover a bound for 
the multiclass hinge loss. A similar, alternative bound for the hybrid loss and an 
extended proof is available in Appendix \ref{sec:pac-bayes-proof}.
\begin{theorem}[Generalisation Margin Bound]
\label{thm:pac-bayes-hybrid-single} For any data distribution $D$,
for any prior $P$ over $w$, for any $w$, any $\delta \in (0,1]$ and
for any $\gamma > 0$ and any $\alpha\in(0,1]$, with probability at least $1-\delta$ over
random samples $S$ from $D$ with $m$ instances, there exists a constant $c$, such that
\begin{align*}
&e_D \leq \PP_{(x,y)\sim S}(\EE_Q (M(w',y)) \leq \gamma) +O\left(\sqrt{\frac{\frac{||w||^2}{2c(1-\alpha)\gamma^{2} } \ln(m|\Ycal|)
+\ln m + \ln \delta^{-1} }{m}}\right).
\end{align*}
\end{theorem}
\begin{proof} [sketch]
By choosing the weight prior $P(w) =
\frac{1}{Z}\exp(-\frac{\|w\|^2}{2})$ and the posterior $Q(w') =
\frac{1}{Z}\exp(-\frac{\|w'-w\|^2}{2})$, one can show $e_D = \PP_{D}
(\EE_Q M(w',y) \leq 0)$ by symmetry argument proposed in
~\cite{LanSha03,McAllester07}. Applying the PAC-Bayes margin
bound~\cite{LanSeeMeg01,ZhuXin09} and knowing the margin threshold $\gamma' \leq c(1-\alpha)\gamma$ and $\text{KL}(Q||P)
= \frac{||w||^2}{2}$ yields the theorem.
\end{proof}

Setting $\alpha=0$ in the above bound recovers a margin bound for SVMs (see \cite{LanSeeMeg01} for an averaging classifiers of SVMs, and \cite{ZhuXin09} for structured case).
Unfortunately, one cannot set $\alpha=1$ to achieve a PAC-Bayes bound for a pure
log loss classifier in this manner due the the $(1-\alpha)^{-1}$ dependence. 
However, to our knowledge, we are not aware of any PAC-Bayes bound on the 
generalisation error for log loss.

%

\section{Experiments}\label{sec:experiments}

\begin{wrapfigure}{r}{0.4\linewidth}
 \centering
     \includegraphics[width=0.9\linewidth]{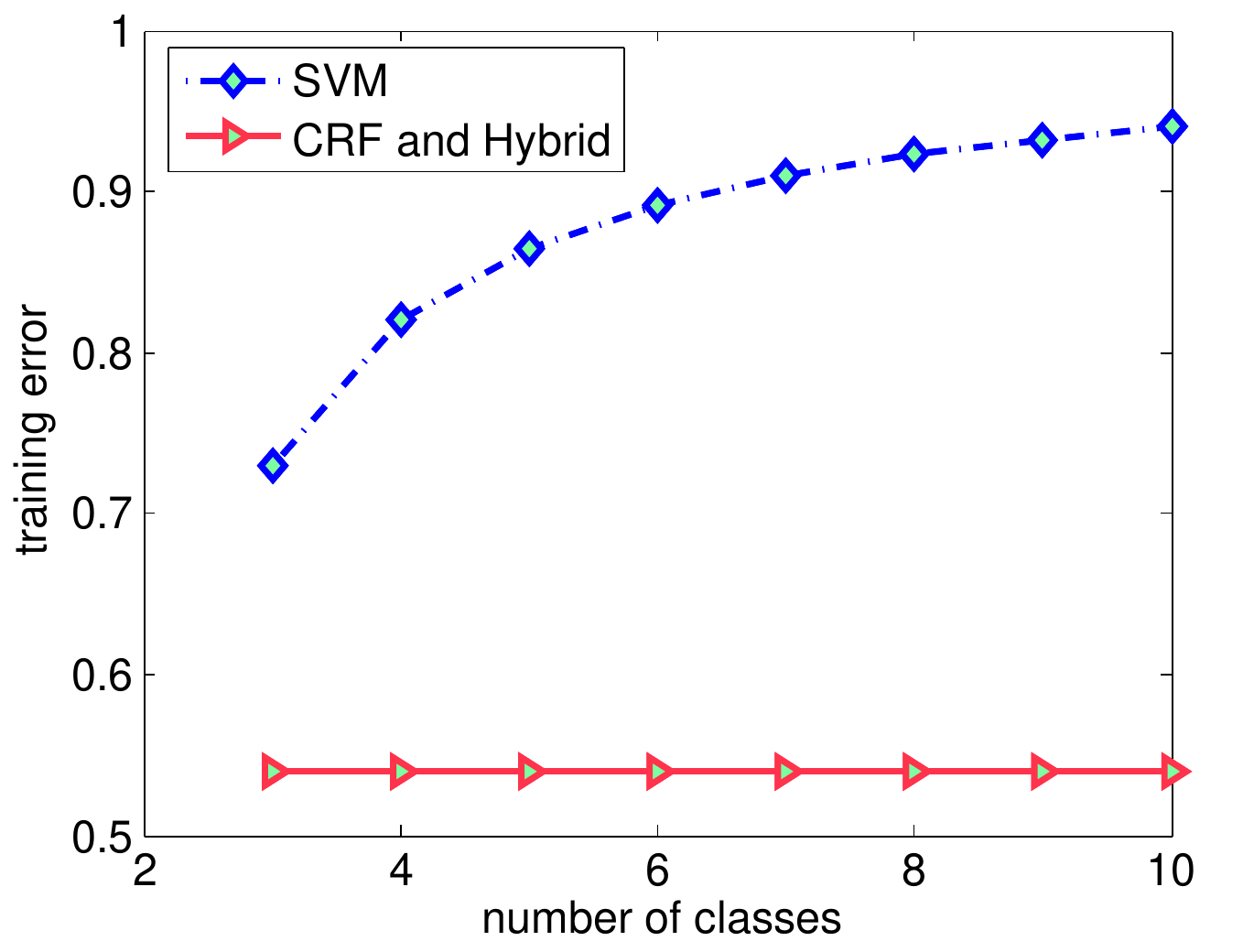}
  \caption{Training Error for
       the hybrid ($\alpha = 0.5$), log and hinge loss vs. number of classes 
       in non-dominant case.} \label{fig:multiclass}
\end{wrapfigure}

The analysis of the hybrid loss suggests it should be able to
outperform the hinge loss due to its improved consistency on distributions with 
non-dominant labels. Furthermore, it should also make more efficient use of 
data than log loss on distributions with dominant labels.
These hypotheses were confirmed by applying the hybrid, log and hinge losses to a 
number of synthetic multiclass data sets in which the data set size and proportion 
of examples with non-dominant labels are carefully controlled. 

We also compared the hybrid loss with the log and hinge losses on
several real structured estimation problems and observed that the hybrid loss
regularly outperforms the other losses and consistently performs at least as well
as the better of the log and hinge losses on any problem.

\subsection{Multiclass Classification}
Two types of multiclass simulations were performed. The first examined the 
performances of the hybrid, log and hinge losses when no observations had a 
dominant label. That is all observations were drawn from a $D$ with $D_y(x) <1/2$ 
for all labels $y$. The second experiment considered distributions with a
controlled mixture of observations with dominant and non-dominant labels.

\paragraph{Non-dominant Distributions} 
To make the experiment as simple as possible, we considered an observation space of 
size $|\Xcal|=1$ and focused on varying the number of labels and their 
probabilities. 
The label set $\Ycal$ took the sizes $|\Ycal|=3,4,5,\dots,10$.
One label $y^*\in\Ycal$ was assigned probability $D_{y^*}(x) = 0.46$ and the 
remainder are given an equal portion of 0.54 (\emph{e.g.}, in the 3 class case the 
other labels each have probability 0.27, and in the 10 class case, 0.06).
Note that this means for all the label set sizes, the gap 
$D_{y^*}(x) - D_{y}(x)$ is at least 0.19 which is always greater than 
$(1-\alpha)(1-2D_{y^*}(x)) = 0.04$ so the hybrid consistency condition 
(\ref{eq:alpha}) is always met.

Features were a constant value in $\RR^2$ as were the parameter vectors $w_y 
\in\RR^2$ for $y\in\Ycal$. Models were found using LBFGS \cite{Byretal94}.
The resulting training errors for hinge, log and hybrid losses are
plotted in Figure~\ref{fig:multiclass} as a function of the number of labels.
As we can clearly see, the hinge loss error increases as the number
of classes increases, whereas the errors for the log and the hybrid
losses remain a constant ($1-D_{y^*}(x)$), in concordance with the consistency 
analysis.

\paragraph{Mix of Non-dominant and Dominant Distributions}
The second synthetic experiment examined how the three losses performed 
given various training set sizes (denoted by $m$) and various proportions of
instances with non-dominant distributions (denoted by $\rho$).

We generated 60 different data sets, all with $\Ycal=\{1,2,3,4,5\}$, 
in the following manner: 
Instances came from either a non-dominant class distribution 
or a dominant class distribution. In the non-dominant class case,
$x\in\RR^100$ is set to a predefined, constant, non-zero vector 
and its label distribution is $D_{1}(x) = 0.4$ and $D_y(x) = 0.15$ for 
$y > 1$. 
In the dominant case, each dimension $x_i$ was drawn from a
normal distribution $N(\mu = 1+y,\sigma = 0.6)$ depending on the class 
$y = 1,\dots, 5$. 
The proportion $\rho$ ranged over 10 values $\rho=0.1, 0.2, 0.3,\dots,1$ and for 
each $\rho$, test and validation sets of size 1000 were generated.
Training set sizes of $m = 30, 60, 100, 300, 600, 1000$ were used for each
$\rho$ value for a total of 60 training sets.
The optimal regularisation parameter $\lambda$ and hybrid loss parameter $\alpha$ 
were selected using the validation set for each loss on each training set.
Then models with parameters $w_y\in\RR^{100}$ for $y\in\Ycal$ were found using 
LBFGS \cite{Byretal94} for each of the three losses on each of the 60 training sets and then 
assessed using the test set.

The results are summarised in Figure~\ref{fig:multiclass_mix}. 
Each point shows the test accuracy for a pair of losses. 
The predominance of points above the diagonal lines in a) and b) show that the 
hybrid loss outperforms the hinge loss and the log loss in most of the data sets. 
while the log and hinge losses perform competitively against each other.

\begin{figure*}[tb]
 \centering
 \subfigure[Hybrid v.s. Hinge (31/15)]{
     \includegraphics[width=0.3\linewidth]{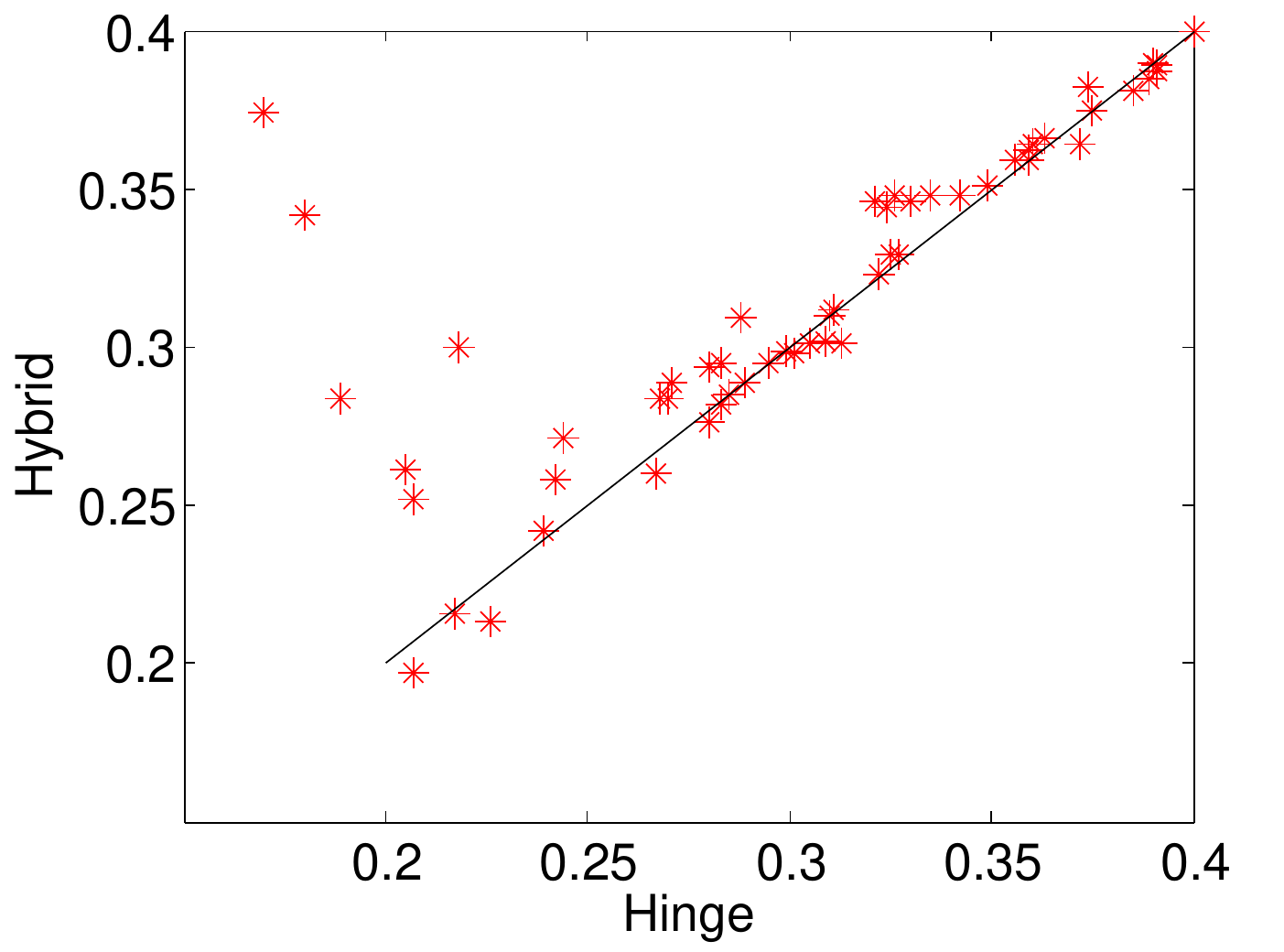}
     \label{fig:hybrid_svm}
     }
      \subfigure[Hybrid v.s. Log (34/15)]{
     \includegraphics[width=0.3\linewidth]{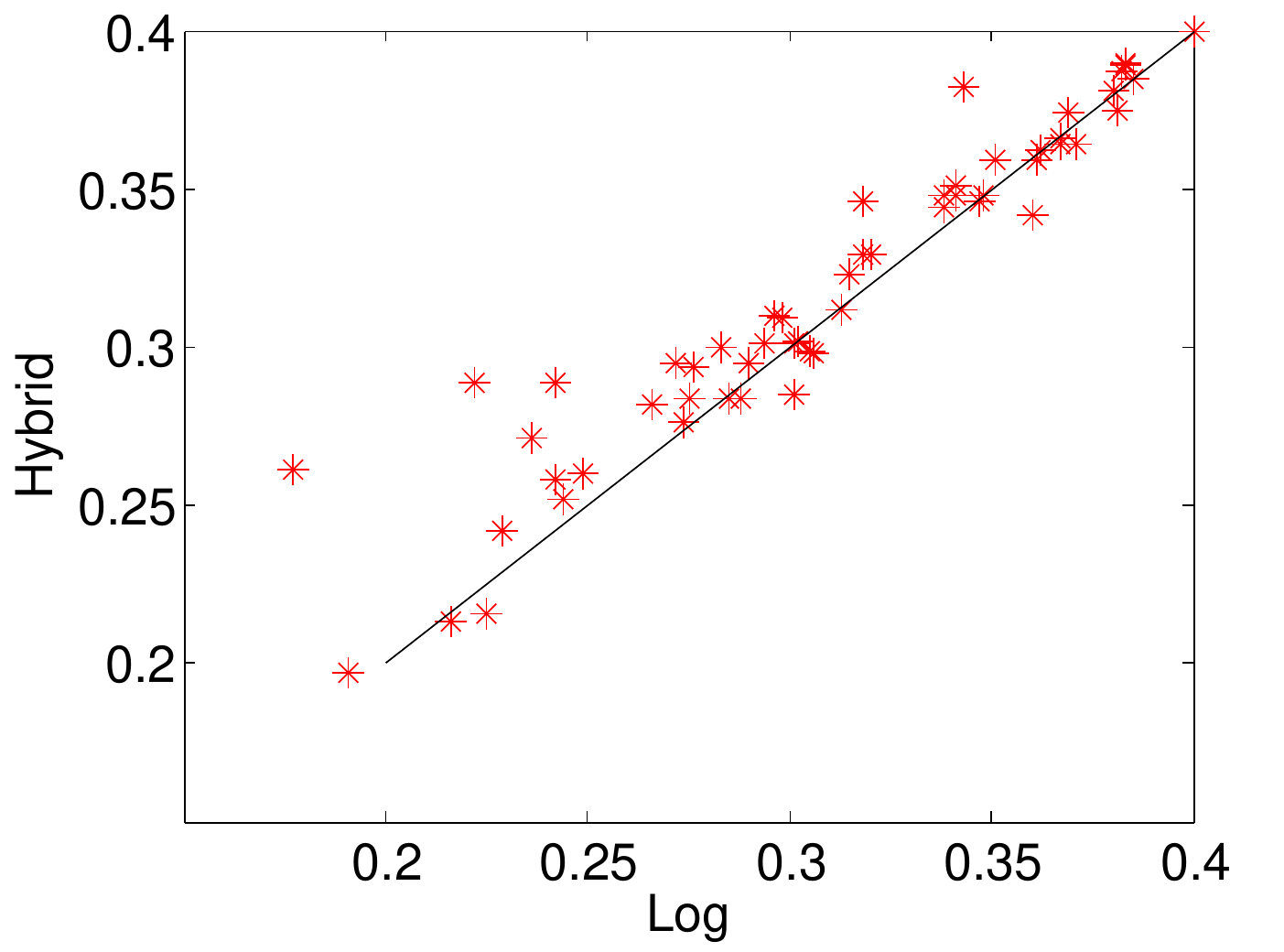}
     \label{fig:hybrid_crf}
     }
      \subfigure[Hinge v.s. Log (30/23)]{
     \includegraphics[width=0.3\linewidth]{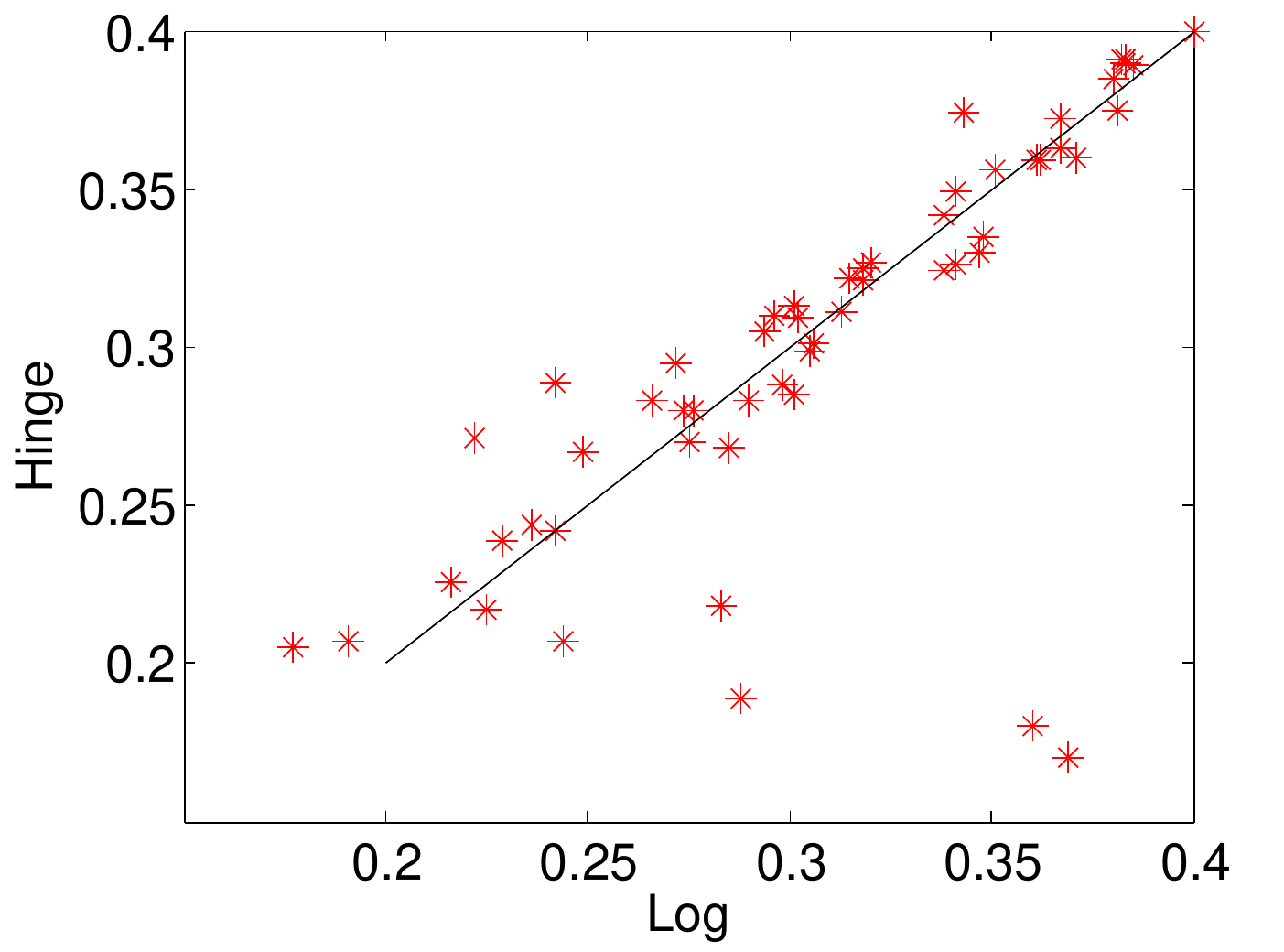}
     \label{fig:svm_crf}
     }
  \caption{Performance of the hybrid, hinge, and log losses on 
  non-dominant/dominant mixtures.
  Points denote pairs of test accuracies for models trained on one of 60 data sets
  using the losses named on the axes.
  Score $(a/b)$ denotes the vertical loss with $a$ wins and $b$ losses (ties
  not counted).}
\label{fig:multiclass_mix}
\end{figure*}

\subsection{Structured Estimation}
Unlike the general multiclass case, structured estimation problems
have a higher chance of non-dominant distributions because of the
very large number of labels as well as ties or ambiguity regarding
those labels. For example, in text chunking,
changing the tag one phrase while leaving the rest unchanged
should not drastically change the probability predictions --
especially when there are ambiguities.
Because of the prevalence of non-dominant distributions, we expect
that training models using a hinge loss to perform poorly on these
problems relative to training with hybrid or log losses.

\paragraph{CONLL2000 Text Chunking} Our first structured estimation experiment 
is carried out on the CONLL2000 text chunking task \cite{conll2000}.
The data set has 8936 training sentences and 2012 testing sentences
with 106978 and 23852 phrases (a.k.a. chunks) respectively. The
task is to divide a text into syntactically correlated parts of words
such as noun phrases, verb phrases, and so on. For a sentence with
$L$ chunks, its label consists of the tagging sequence of all its
chunks, \emph{i.e.} $y = (y^1,y^2,\dots,y^L)$, where $y^i$ is the chunking
tag for chunk $i$. As commonly used in this task, the label $y$ is
modelled as a 1D Markov chain to account for the dependency between
adjacent chunking tags $(y_i^j,y_i^{j+1})$ given observation $x_i$.
Clearly, the model has exponentially many possible labels, which
suggests there are many non-dominant classes.

\begin{figure}[tb]
 \centering
 \subfigure[the testing
  set ]{
     \includegraphics[width=0.46\linewidth]{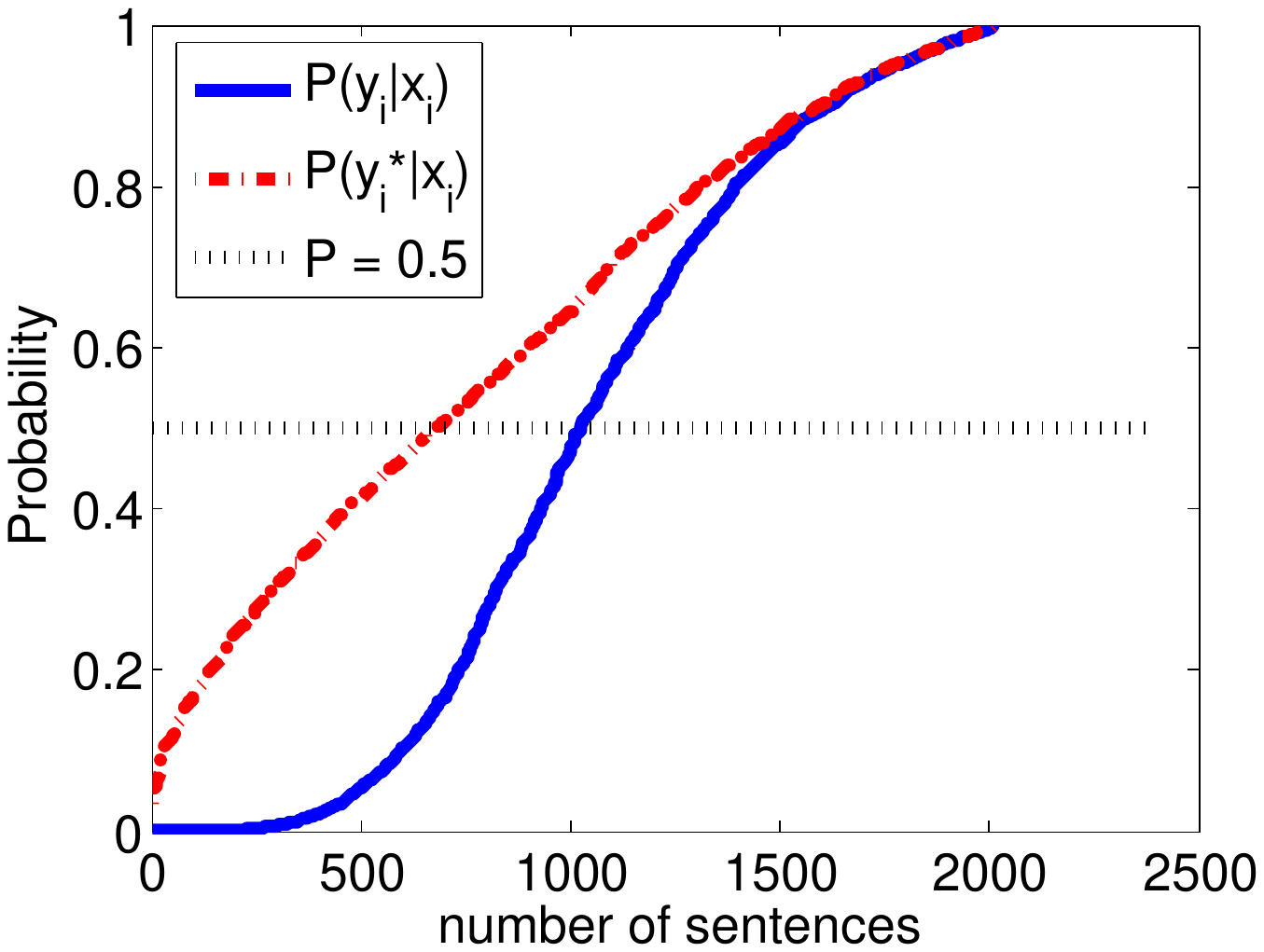}
     \label{fig:chunk_test}
     }
      \subfigure[the training set]{
     \includegraphics[width=0.46\linewidth]{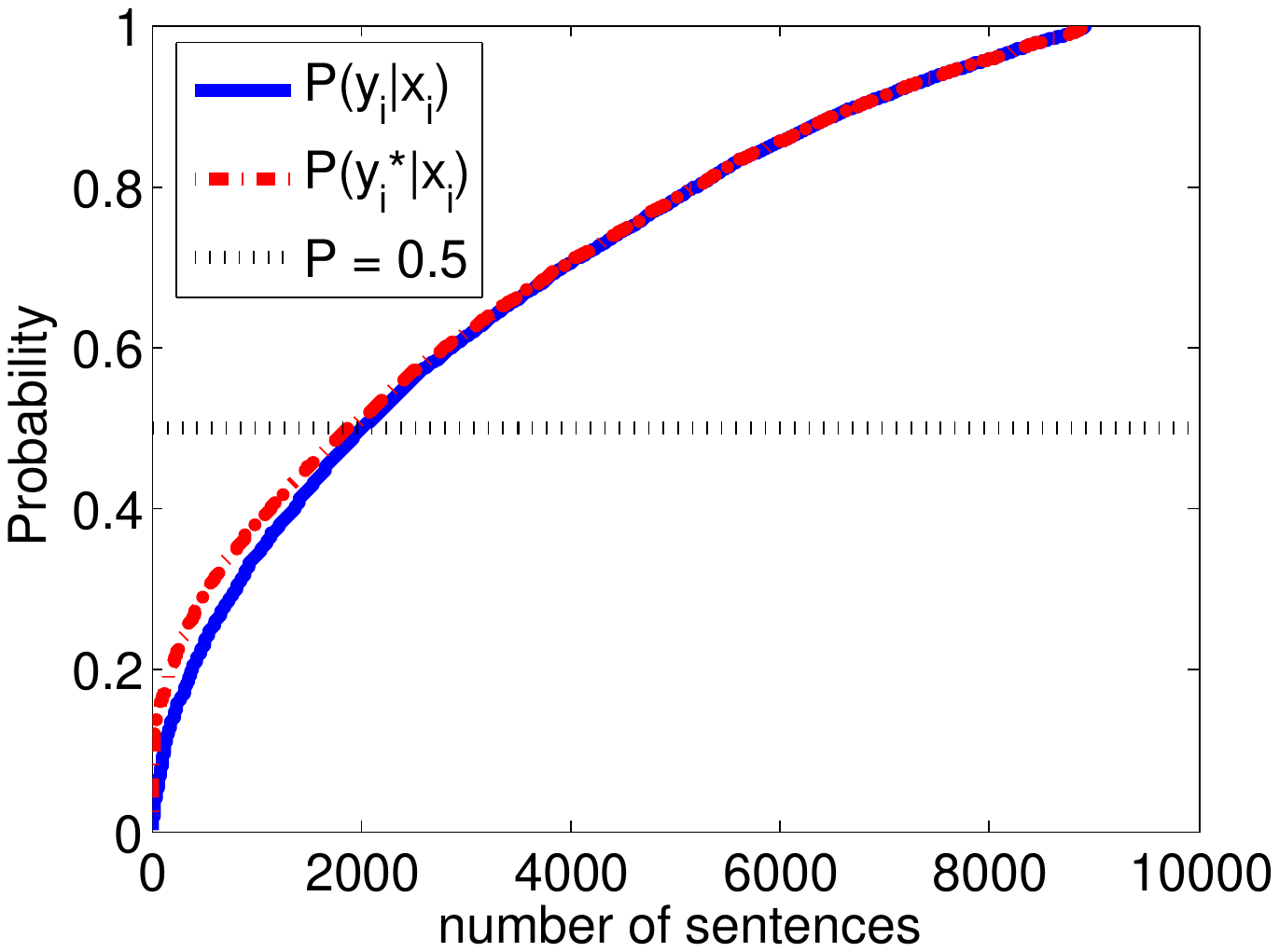}
     \label{fig:chunk_train}
     }
  \caption{Estimated probabilities of the true label $D_{y_i}(x_i)$ and most likely label $D_{y_i^*}(x_i)$. Sentences are sorted according to $D_{y_i}(x_i)$ and $D_{y_i^*}(x_i)$ respectively in ascending
  order. $D=1/2$ is shown as the straight black dot line. About 700
  sentences out of 2012 in the testing set and 2000 sentences out of 8936 in the training set
  have no dominant class.
} \label{fig:chunking}
\end{figure}

\begin{table*}[t]
    \centering {\caption{Accuracy, precision, recall and F1 Score on the
    CONLL2000 text chunking task.} \label{table:conll2000}}
    \begin{tabular}{c|c|c|c|c|c}
    Train Portion &Loss &Accuracy &Precision & Recall &F1 Score\\
    &Hinge&91.14 &85.31 &85.52 &85.41\\
    0.1&Log&92.05 &87.04 &87.01 &87.02\\
    &Hybrid&92.07 &87.17 &86.93 &87.05\\
    \hline
    &Hinge&94.61 &91.23 &91.37 &91.30\\
    1&Log&95.10 &92.32 &91.97 &92.15\\
    &Hybrid&95.11 &92.35 &92.00 &92.17\\
    \end{tabular}
\end{table*}

Since the true underlying distribution is unknown, we train a
CRF\footnote{
    using the feature template from the CRF++ toolkit \cite{crfplusplus}, and
    the CRF code from Leon Bottou \cite{crfsgd}.} 
on the training set and then apply
the trained model to both testing and training datasets to get an
estimate of the conditional distributions for each instance. We sort
the sentences $x_i$ from highest to lowest estimated probability on
the true chunking label $y_i$ given $x_i$. The result is plotted in
Figure~\ref{fig:chunking}, from which we observe the existence of
many non-dominant distributions --- about 1/3 of the testing
sentences and about 1/4 of the training sentences.

We split the data into 3 parts: training ($20\%$), testing ($40\%$)
and validation ($40\%$). The regularisation parameter $\lambda$ and
the weight $\alpha$ were determined via parameter selection using the
validation set. To see the performance with different training
sizes, we took part of the training data to learn the model and gathered statistics
on the test set. The accuracy, precision, recall
and F1 Score on test set are reported in Table~\ref{table:basenp}
when using 10\% and 100\% of the training set. 
The hybrid loss outperforms both the hinge loss and the log loss (albeit 
marginally). 

\begin{table*}[t]
\centering {\caption{Accuracy, precision, recall and F1 Score on the
baseNP chunking task for training on increasing portions of training
set.} \label{table:basenp}
\begin{tabular}{c|c|c|c|c|c}
Train Portion &Loss &Accuracy &Precision & Recall &F1 Score\\
&Hinge&88.48 &71.70 &75.96 &73.77\\ 
0.1&Log&90.86 &81.09 &78.96 &80.01\\
&Hybrid&90.90 &81.23 &79.09 &80.15\\
\hline
&Hinge&94.64 &87.58 &88.30 &87.94\\
1& Log&95.21 &90.07 &88.89 &89.48\\
&Hybrid&95.24 &90.12 &88.98 &89.55\\
\end{tabular}
}\\
\centering {\caption{Accuracy, precision, recall and F1 Score on the
Japanese named entity recognition task.} \label{table:japan}
\begin{tabular}{l|c|c|c|c}
Loss &Accuracy &Precision & Recall &F1 Score\\
Hinge&95.63 &73.24 &64.37 &68.52 \\
Log&95.92 &78.22 &64.85 &70.91\\
Hybrid&95.95 &79.02 &65.32 &71.52\\
\end{tabular} } \centering
\end{table*}
\paragraph{baseNP Chunking}
A similar methodology to the previous experiment is applied to the BaseNP data set  
\cite{crfplusplus}. It has 900 sentences in
total and the task is to automatically classify a chunking phrase is as
baseNP or not. We split the data into 3 parts: training ($20\%$),
testing ($40\%$) and validation ($40\%$). Once again, $\lambda$ and $\alpha$ are
determined via model selection on the validation set. We report the
test accuracy, precision, recall and F1 Score in
Table~\ref{table:basenp} for training on increasing proportion of
the training set. The hybrid outperforms the other two losses on all
measures.

\paragraph{Japanese named entity recognition}
Finally, we used a multiclass data set containing 716 Japanese sentences and 17 
annotated named entities \cite{crfplusplus}. 
The task is to locate and classify proper nouns and
numerical information in a document into certain classes of named
entities such as names of persons, organizations, and locations. We
train all 3 models on 216 sentences and test on 500 sentences with
the default parameters found in Bottou's CRF code. The extra parameter
$\alpha$ is selected for the smallest test error. The result is
reported in Table~\ref{table:japan}. Once again, the hybrid loss outperforms the
others two losses.

\section{Conclusion and Discussion}\label{sec:conclusions}

We have provided theoretical and empirical motivation for the use of
a novel hybrid loss for multiclass and structured prediction
problems which can be used in place of the more common log loss or
multiclass hinge loss. This new loss attempts to blend the strength
of purely discriminative approaches to classification, such as
Support Vector machines, with probabilistic approaches, such as
Conditional Random Fields. Theoretically, the hybrid loss enjoys
better consistency guarantees than the hinge loss while
experimentally we have seen that the addition of a purely
discriminative component can improve accuracy when data is less
prevalent.

\subsection{Future Work}
Theoretically, we expect that some stronger sufficient conditions on $\alpha$ are
possible since the bounds used to establish Theorem~\ref{thm:fcc}
are not tight. Our conjecture is that a necessary and sufficient
condition would include a dependency on the number of classes.
We are also investigating connections between $\alpha$ and the multiclass Tsybakov 
noise condition \cite{Che06}.

To our knowledge, the notion of a regular function class for the purposes of
consistency analysis is a novel one.
Characterisations of this property for various existing parametric models would make 
testing for regularity easier.

One current limitation of the hybrid model is the use of a single, fixed $\alpha$ for 
all observations in a training set. 
One interesting avenue to explore would be trying to
dynamically estimate a good value of $\alpha$ on a per-observation
basis. This may further improve the efficacy of the hybrid loss by
exploiting the robustness of SVMs (low $\alpha$) when the label
distribution for an observation has a dominant class but switching
to probability estimation via CRFs (high $\alpha$) when this is not
the case.


\small \baselineskip=3.7mm
\bibliography{bibfile,smmcrf}
\bibliographystyle{plain}

\appendix
\section{Proof for Consistency}\label{sec:proof-fcc}
\begin{proof}
    {\bf of Theorem~\ref{thm:fcc}}
    We use $L_\alpha(p,D) = \EE_{y\sim D}\left[\ell\alpha(p,y)\right]$
    and $\Delta(\Ycal)$ to denote distributions over $\Ycal$.
    Since we a free to permute labels within $\Ycal$ we will assume without loss
    of generality that $D_1 = \max_{y\in\Ycal} D_y$ and
    $D_2 = \max_{y\ne 1} D_y$.
    The proof now proceeds by contradiction and assumes there is some minimiser
    $p = \argmin_{q\in\Delta(\Ycal)} L_\alpha(q,D)$ that is not aligned with
    $D$.
    That is, there is some $y^*\ne 1$ such that $p_{y^*} \ge p_1$.
    For simplicity, and again without loss of generality, we will assume
    $y^*=2$.

    The first case to consider is when $p_2$ is a maximum and $p_1 < p_2$.
    Here we construct a $q$ that ``flips'' the values of $p_1$ and $p_2$ and
    leaves all the values unchanged.
    That is, $q_1 = p_2$, $q_2 = p_1$ and $q_y = p_y$ for all $y=3,\ldots,k$.
    Intuitively, this new point is closer to $D$ and therefore the CRF component
    of the loss will be reduced while the SVM loss won't increase.
    The difference in conditional risks satisfies
    \begin{eqnarray}
        L_\alpha(p,D) - L_\alpha(q,D)
        &=& \sum_{y=1}^k D_y.(\ell_\alpha(p,y) - \ell_\alpha(q,y)) \nonumber \\
        &=& D_1.(\ell_\alpha(p,1) - \ell_\alpha(q,1)) \nonumber\\
        &&  + D_2.(\ell_\alpha(p,2) - \ell_\alpha(q,2)) \nonumber \\
        &=& (D_1 - D_2)(\ell_\alpha(q,2) - \ell_\alpha(q,1)) \nonumber
    \end{eqnarray}
    since $\ell_\alpha(p,1)=\ell_\alpha(q,2)$ and
    $\ell_\alpha(p,2) = \ell_\alpha(q,1)$ and the other terms cancel by
    construction.
    As $D_1 - D_2 > 0$ by assumption, all that is required now is to show
    that
    $\ell_\alpha(q,2) - \ell_\alpha(q,1) = \alpha\ln\frac{q_1}{q_2}
    +(1-\alpha)(\ell_H(q,2)-\ell_H(q,1))$
    is strictly positive.

    Since $q_1 > q_y$ for $y\ne 1$ we have $\ln\frac{q_1}{q_2} > 0$,
    $\ell_H(q,2) = \left[1-\ln\frac{q_2}{q_1}\right]_+ > 1$, and
    $\ell_H(q,1) = \left[1-\ln\frac{q_1}{q_y}\right]_+ < 1$, and so
    $\ell_H(q,2) - \ell_H(q,1) > 1 - 1 = 0$.
    Thus, $\ell_\alpha(q,2) - \ell_\alpha(q,1) > 0$ as required.

    Now suppose that $p_2 = p_1$ is a maximum.
    In this case we show a slight perturbation
    $q = (p_1 + \epsilon, p_2 - \epsilon, p_3, \ldots, p_k)$ yields a lower
    for $\epsilon >0$.
    For $y\ne 1,2$ we have $\ell_L(p,y) - \ell(q,y) = 0$ and since $p_2 > p_y$
    and $q_1 > q_y$ thus
    $\ell_H(p,y) - \ell_H(q,y)
    = 1-\ln\frac{p_y}{p_2} + 1 -\ln\frac{q_y}{q_1}
    = \ln\frac{p_2}{q_1}
    > 1-\frac{q_1}{p_2}
    = -\frac{\epsilon}{p_2}$
    since $-\ln x > 1 - x$ for $x\in(0,1)$ and
    $q_1 = p_1 + \epsilon = p_2 + \epsilon$.
    Therefore
    \begin{equation}\label{eq:yne12}
        \ell_\alpha(p,y) - \ell_\alpha(q,y) > -\epsilon\frac{(1-\alpha)}{p_1}
    \end{equation}

    When $y=1$,
    $\ell_L(p,1) - \ell_L(q,1) = -\ln\frac{p_1}{q_1} > \frac{q_1 - p_1}{p_1}
    = \frac{\epsilon}{p_1}$ and
    $\ell_H(p,1)-\ell_H(q,1) = (1-\ln\frac{p_1}{p_2}) - (1-\ln\frac{q_1}{q_2}) =
    \ln\frac{q_1}{q_2} = \ln\frac{p_1+\epsilon}{p_1-\epsilon}$ since
    $p_1 = p_2$.
    Thus $\ell_H(p,1)-\ell_H(q,1) > 1 - \frac{p_1 - \epsilon}{p_1 + \epsilon}
    = \frac{2\epsilon}{p_1+\epsilon}$.
    And so
    \begin{equation}\label{eq:y1}
        \ell_\alpha(p,y) - \ell_\alpha(q,y)
        > \epsilon\left[\frac{\alpha}{p_1}
            + \frac{2(1-\alpha)}{p_1+\epsilon}\right]
    \end{equation}
    Finally, when $y=2$ we have $\ell_L(p,2)-\ell_L(q,2) = -\ln\frac{p_2}{q_2}
    > \frac{q_2 - p_2}{p_2} = \frac{-\epsilon}{p_1}$ and
    $\ell_H(p,2) - \ell_H(q,2) = (1-\ln\frac{p_2}{p_1})-(1-\ln\frac{q_2}{q_1})
    = \ln\frac{q_2}{q_1} > 1 - \frac{q_1}{q_2} =
    \frac{-2\epsilon}{p_1+\epsilon}$.
    Thus,
    \begin{equation}\label{eq:y2}
        \ell_\alpha(p,2) - \ell_\alpha(q,2)
        > -\epsilon\left[
            \frac{\alpha}{p_1} + \frac{2(1-\alpha)}{p_1+\epsilon}
        \right].
    \end{equation}

    Putting the inequalities (\ref{eq:yne12}), (\ref{eq:y1}) and (\ref{eq:y2})
    together yields
    \begin{eqnarray*}
        \lefteqn{\lim_{\epsilon\to 0} \frac{L_\alpha(p,D) - L_\alpha(q,D)}{\epsilon}} \\
        &>& \lim_{\epsilon\to 0}
            (D_1-D_2)
            \left[\frac{\alpha}{p_1}+\frac{2(1-\alpha)}{p_1+\epsilon}\right]
            - \sum_{y=3}^k D_y\frac{1-\alpha}{p_1} \\
        &=& \frac{D_1 - D_2}{p_1}(2-\alpha) - \frac{1-D_1-D_2}{p_1}(1-\alpha) \\
        &=& \frac{1}{p_1}(D_1 - D_2 + (1-\alpha)(2D_1 - 1)).
    \end{eqnarray*}
    Observing that since $D_1 > D_2$,
    when $D_1 > \frac{1}{2}$ the final term is positive without any constraint
    on $\alpha$
    and when $D_1 < \frac{1}{2}$ the difference in risks is positive
    whenever
    \begin{equation}
        \alpha > 1 - \frac{D_1 - D_2}{1 - 2D_1}
    \end{equation}
    completes the proof.
\end{proof}

\section{Proof of Necessity of FCC}
\label{sec:proof-fcc-nec}

\begin{proof}
    {\bf of Theorem~\ref{thm:fcc-necessary}}
    The proof is by contradiction.
    We assume we have a regular function class $\Fcal$ and a loss $\ell$ which is
    $\Fcal$-consistent but not FCC.
    That is, (\ref{eq:F-consistency}) holds for $\ell$ but there exists a
    distribution $p$ over $\Ycal$
    such that there is a $g \in \RR^{\Ycal}$ which minimises the conditional risk
    $L_{q}(g)$ but $\argmax_{y\in\Ycal} g_y \ne \argmax_{y\in\Ycal} q_y$.

    By the assumption of the regularity of $\Fcal$ there is an $x\in\Xcal$ and a
    $f\in\Fcal$ so that $f(x) = g$.
    We now define a distribution $D$ over $\Xcal\times\Ycal$ that puts all its mass
    on the set $\{x\}\times\Ycal$ so that $D(x,y) = p_y$.
    Since this distribution is concentrated on a single $x$ its full risk and
    conditional risk on $x$ are the same. That is, $L_D(\cdot) = L_p(\cdot)$.
    Thus,
    \[
        L_D(f) = L_p(f) = \inf_{f'\in\Fcal} L_p(f') = \inf_{f'\in\Fcal} L_D(f')
    \]
    By the assumption of $\Fcal$-consistency, since $f$ is a minimiser of $L_D$ it
    must also minimise $e_D$.
    Once again, the construction of $D$ means that
    $e_D(f) = e_p(g) = \PP_{y\sim p}\left[y \ne \argmax_{y'\in\Ycal}g_y\right] =
    1 - p_{y_g}$ where $y_g = \argmax_y g_y$ is the label predicted by $g$.
    However,
    \[
        e_D(f) = e_p(g) = 1-p_{y_g} > 1 - p_{y^*}
    \]
    since $y_* = \argmax_y p_y \ne \argmax g_y = y_g$.

    By the second regularity property, there must also be an $\hat{f}\in\Fcal$ such
    that $\argmax_y \hat{f}_y(x) = y^*$ so that
    $e_D(f) > \inf_{f'\in\Fcal} e_D(f') = e_D(\hat{f}) = 1 - p_{y^*}$.
    Thus, we have shown that there exists a distribution $D$ so $f\in\Fcal$ is a
    minimiser of the risk $L_D$ but is not a minimiser of the
    misclassification rate $e_D$ which contradicts the assumption of the
    $\Fcal$-consistency of $\ell$. Therefore, $\ell$ must be FCC.

\end{proof}

\section{Proof for PAC-Bayes Bounds}
\label{sec:pac-bayes-proof}

For explicitly, we rewrite $M$ and $p_y$ as $M(x,y;w)$ and
$p(y|x;w)$ when they are parameterized by $w$.

\begin{theorem}[Generalisation Bound]
\label{thm:pac-bayes-coupling-single} For any data distribution $D$,
for any prior $P$ over $w$, for any $\delta \in (0,1]$ and $\alpha
\in [0,1)$ and for any $\gamma \geq 0$, for any $w$, with
probability at least $1-\delta$ over random samples $S$ from $D$
with $m$ instances, we have
\begin{align*}
\nonumber \EE_{D}\Big[\Big(\gamma-M(x,y;w)\Big)_+\Big] \leq
\frac{1}{m}\sum_{i=1}^m\Big(\gamma-M(x_i,y_i;w)\Big)_+\\\nonumber
+\frac{1}{(1-\alpha)}\left(\alpha\sqrt{\frac{1}{m}} +
\sqrt{\frac{\ln \frac{1}{P(w)}+\ln A(\alpha,w)+\ln
\frac{1}{\delta(1-e^{-2})}}{2m}}\right),
\end{align*}
where \begin{align*} R(\alpha,w)
=\alpha\EE_{D}\Big[-\ln{p(y|x;w)}\Big]+(1-\alpha)\EE_{D}\Big[\Big(\gamma-M(x,y;w)\Big)_+\Big],\\
R_S(\alpha,w) =\Big[\alpha\frac{\sum_{i=1}^m -\ln{p(y_i|x_i;w)}}{m}
+(1-\alpha) \frac{\sum_{i=1}^m \Big(\gamma-M(x_i,y_i;w)\Big)_+}{m}\Big],\\
A(\alpha,w) = \EE_{s\sim
D^m}e^{2m(R(\alpha,w)-R_S(\alpha,w))^2}.\end{align*}
\end{theorem}

Here $A$ is upper bounded independently of $D$. For example, for a
zero-one loss, it is upper bounded by $m+1$ (see
\cite{GerLacLavMar08}). The theorem gives a bound on the true margin
error of the hybrid model. The theorem follows theorem
\ref{thm:pac-bayes-coupling} in the appendix immediately.

\begin{lemma}[PAC-Bayes bound\cite{McAllester01,GerLacLavMar08}]
\label{thm:pac-bayes-gibbs} For any data distribution $D$, for any
prior $P$ and posterior $Q$ over $w$, for any $\delta \in (0,1]$,
for any loss $\ell$. With probability at least $1-\delta$ over
random sample $S$ from $D$ with $m$ instances, we have
\begin{align}
\nonumber R(Q,\ell) \leq R_S(Q,\ell) +
\sqrt{\frac{\text{KL}(Q||P)+\ln(\frac{1}{\delta}\EE_{s\sim D^m}\EE_{w\sim
P}e^{2m(R(Q,\ell)-R_S(Q,\ell))^2})}{2m}},
\end{align}
where $\text{KL}(Q||P) := \EE_{w\sim Q} \ln(\frac{Q(w)}{P(w)}) $ is the
Kullback-Leibler divergence between $Q$ and $P$, and $R(Q,\ell)
=\EE_{Q,D}[\ell(x,y;w)], R_S(Q,\ell) =\EE_{Q}\frac{\sum_{i=1}^m
\ell(x_i,y_i;w)}{m}.$
\end{lemma}

\begin{theorem}[Bound on Averaging classifier]
\label{thm:pac-bayes-coupling} For any data distribution $D$, for
any prior $P$ and posterior $Q$ over $w$, for any $\delta \in (0,1]$
and $\alpha \in [0,1)$ and for any $\gamma \geq 0$. With probability
at least $1-\delta$ over random sample $S$ from $D$ with $m$
instances, we have
\begin{align*}
\nonumber \EE_{Q,D}\Big[[\gamma-M(x,y;w)]_+\Big] \leq
\frac{1}{m}\EE_{Q}\Big[\sum_{i=1}^m[\gamma-M(x_i,y_i;w)]_+\Big] \\
+\frac{\alpha}{1-\alpha}\sqrt{\frac{1}{m}} +
\frac{1}{1-\alpha}\sqrt{\frac{\text{KL}(Q||P)+\ln A(\alpha)+\ln
\frac{1}{\delta(1-e^{-2})}}{2m}},
\end{align*}
where $\text{KL}(Q||P) := \EE_{w\sim Q} \ln(\frac{Q(w)}{P(w)}) $ is the
Kullback-Leibler divergence between $Q$ and $P$, and \begin{align*}
R(\alpha)
=\alpha\EE_{Q,D}\Big[-\ln{p(y|x;w)}\Big]+(1-\alpha)\EE_{Q,D}\Big[\Big(\gamma-M(x,y;w)\Big)_+\Big],\\
R_S(\alpha) =\EE_{Q}\Big[\alpha\frac{\sum_{i=1}^m
-\ln{p(y_i|x_i;w)}}{m}
+(1-\alpha) \frac{\sum_{i=1}^m \Big(\gamma-M(x_i,y_i;w)\Big)_+}{m}\Big],\\
A(\alpha) = \EE_{s\sim D^m}\EE_{w\sim
P}e^{2m(R(\alpha)-R_S(\alpha))^2}.\end{align*}
\end{theorem}
\begin{proof}
Since $\EE_{D}\left(\EE_{Q}\Big[\frac{\sum_{i=1}^m -\ln
{p(y_i|x_i;w)}}{m}\Big]\right) = \EE_{Q,D}\Big[-\ln{p(y|x;w)}\Big],
$ by Chernoff bound we have \[ \PP_{S\sim
D^m}\left(\EE_{Q}\left[\frac{\sum_{i=1}^m
-\ln{p(y_i|x_i;w)}}{m}\right] - \EE_{Q,D}\Big[-\ln{p(y|x;w)}\Big] <
\epsilon\right) > 1-e^{-2m\epsilon^2}. \]

Define $B(S) :=\EE_{Q}\left[\frac{\sum_{i=1}^m
-\ln{p(y_i|x_i;w)}}{m}\right] - \EE_{Q,D}\Big[-\ln{p(y|x;w)}\Big]. $

Applying Lemma \ref{thm:pac-bayes-gibbs} for $R(\alpha)$ and
$R_S(\alpha)$, we have for any $P,Q$
\begin{align*}
\delta > & \PP_{S\sim D^m}\left(R(\alpha) \geq
R_S(\alpha)+\sqrt{\frac{\text{KL}(Q||P)+\ln \frac{1}{\delta} + \ln
A(\alpha)}{2m}}\right)\\
\geq &\PP_{S\sim D^m}\left(R(\alpha) \geq
R_S(\alpha)+\sqrt{\frac{\text{KL}(Q||P)+\ln \frac{1}{\delta} + \ln
A(\alpha)}{2m}}, B(S) < \epsilon\right)\\
\geq& \PP_{S\sim
D^m}\left((1-\alpha)\EE_{Q,D}\Big[\Big(\gamma-M(x,y;w)\Big)_+\Big]
\geq (1-\alpha) \frac{\sum_{i=1}^m
\Big(\gamma-M(x_i,y_i;w)\Big)_+}{m}\right.\nonumber
\\&\left.+\alpha\epsilon +\sqrt{\frac{\text{KL}(Q||P)+\ln \frac{1}{\delta} + \ln
A(\alpha)}{2m}}, B(S) < \epsilon\right)\\
=& \PP_{S\sim
D^m}\left((1-\alpha)\EE_{Q,D}\Big[\Big(\gamma-M(x,y;w)\Big)_+\Big]
\geq (1-\alpha) \frac{\sum_{i=1}^m
\Big(\gamma-M(x_i,y_i;w)\Big)_+}{m}\right.\nonumber
\\&\left.+\alpha\epsilon +\sqrt{\frac{\text{KL}(Q||P)+\ln \frac{1}{\delta} + \ln
A(\alpha)}{2m}}\Big|B(S) < \epsilon\right)\PP_{S\sim D^m}\Big(B(S) < \epsilon\Big)\\
&\geq \PP_{S\sim
D^m}\left((1-\alpha)\EE_{Q,D}\Big[\Big(\gamma-M(x,y;w)\Big)_+\Big]
\geq (1-\alpha) \frac{\sum_{i=1}^m
\Big(\gamma-M(x_i,y_i;w)\Big)_+}{m}
\right.\nonumber\\&\left.+\alpha\epsilon +\sqrt{\frac{\text{KL}(Q||P)+\ln
\frac{1}{\delta} + \ln A(\alpha)}{2m}}\right)\PP_{S\sim
D^m}\Big(B(S) < \epsilon\Big)
\end{align*}
Divide two sides by $\PP_{S\sim D^m}\Big(B(S) < \epsilon\Big)$, we
get
\begin{align*}
&\PP_{S\sim
D^m}\left((1-\alpha)\EE_{Q,D}\Big[\Big(\gamma-M(x,y;w)\Big)_+\Big]
\geq (1-\alpha) \frac{\sum_{i=1}^m
\Big(\gamma-M(x_i,y_i;w)\Big)_+}{m}\right.\nonumber
\\&\left.+\alpha\epsilon +\sqrt{\frac{\text{KL}(Q||P)+\ln \frac{1}{\delta} + \ln
A(\alpha)}{2m}}\right) \leq \frac{\delta}{ \PP_{S\sim D^m}\Big(B(S)
<\epsilon\Big)}\leq \frac{\delta}{1-e^{-2m\epsilon^2}}.
\end{align*}
Let $\epsilon = \sqrt{\frac{1}{m}}$, and then let
$\delta'=\frac{\delta}{1-e^{-2m(\epsilon ^2)}}
=\frac{\delta}{1-e^{-2}}$, we get $\delta =
\frac{1}{\delta'(1-e^{-2})}$. The theorem follows by substituting
$\delta$  with $\delta'$ and dividing by $(1-\alpha)$ on both sides
of the inequality inside of the probability.
\end{proof}

\end{document}